\documentclass[journal,twocolumn,singled-spaced]{IEEEtran}

\usepackage{makeidx}  
\usepackage{graphicx} 
\usepackage{subfigure} 
\usepackage{epstopdf}
\usepackage{hyperref}
\usepackage{mathtools}
\usepackage{pgfplots}

\usepackage{url}
\usepackage{lettrine}

\usepackage{algorithm}
\usepackage{algorithmic}

\usepackage{amsmath,amssymb}
\usepackage[utf8]{inputenc}
\usepackage{color}

\usepackage{tikz}
\usetikzlibrary{arrows,backgrounds,snakes}

\newcommand{\bx}{\boldsymbol{x}}

\newcommand{\bX}{\boldsymbol{X}}

\newcommand{\Yhat}{\widehat{Y}}

\DeclarePairedDelimiter\floor{\lfloor}{\rfloor}
\DeclareMathOperator*{\argmax}{argmax}

\DeclareMathOperator*{\hm}{\textsc{hm}}

\newcommand{\field}[1]{\mathbb{#1}}

\newcommand{\R}{\field{R}}

\newcommand{\E}{\field{E}}

\renewcommand{\Pr}{\field{P}}

\newcommand{\Ind}[1]{\field{I}{ \left\{{#1}\right\} }}

\newcommand{\scL}{\mathcal{L}}

\newcommand{\scF}{\mathcal{F}}
\newcommand{\scD}{\mathcal{D}}

\newcommand{\dt}{\displaystyle}

\newcommand{\wh}{\widehat}
\newcommand{\ve}{\varepsilon}
\newcommand{\hPhi}{\wh{\Phi}}
\newcommand{\Dhat}{\hPhi}
\newcommand{\Fhat}{\wh{F}}
\newcommand{\Ghat}{\wh{G}}
\newcommand{\Hhat}{\wh{H}}

\newcommand{\phat}{\wh{p}}
\newcommand{\qhat}{\wh{q}}

\newcommand{\sH}{H_{1/2}}
\newcommand{\sHhat}{\Hhat_{1/2}}

\newcommand{\bool}{\{0,1\}}

\newlength{\minipagewidth}
\newcommand{\bookbox}[1]{
\par\medskip\noindent
\framebox[\linewidth]{
\begin{minipage}{\minipagewidth}
{#1}
\end{minipage} } \par\medskip }

\newlength\figureheightp 
\newlength\figurewidthp 
\newlength\figureheight 
\newlength\figurewidth 
\usetikzlibrary{datavisualization}
\tikzset{every picture/.style={font issue={\fontsize{9}{10}}},font issue/.style={execute at begin picture={#1\selectfont}}}

\newtheorem{lemma}{Lemma}
\newtheorem{theorem}{Theorem}

\newtheorem{remark}{Remark}

\newtheorem{proof}{Proof}
\newtheorem{definition}{Definition}

\begin{document}

\title{Confidence Decision Trees via Online and Active Learning for Streaming (BIG) Data}

\author{Rocco~De~Rosa\\
	{\tt\small derosa@dis.uniroma1.it}}        



\maketitle

\begin{abstract}

 Decision tree classifiers are a widely used tool in data stream mining. The use of confidence intervals to estimate the gain associated with each split leads to very effective methods, like the popular Hoeffding tree algorithm. From a statistical viewpoint, the analysis of decision tree classifiers in a streaming setting requires knowing when enough new information has been collected to justify splitting a leaf. Although some of the issues in the statistical analysis of Hoeffding trees have been already clarified, a general and rigorous study of confidence intervals for splitting criteria is missing. We fill this gap by deriving accurate confidence intervals to estimate the splitting gain in decision tree learning with respect to three criteria: entropy, Gini index, and a third index proposed by Kearns and Mansour. Our confidence intervals depend in a more detailed way on the tree parameters. We also extend our confidence analysis to a selective sampling setting, in which the decision tree learner adaptively decides which labels to query in the stream. We furnish theoretical guarantee bounding the probability that the classification is non-optimal learning the decision tree via our selective sampling strategy. Experiments on real and synthetic data in a streaming setting show that our trees are indeed more accurate than trees with the same number of leaves generated by other techniques and our active learning module permits to save labeling cost. In addition, comparing our labeling strategy with recent methods, we show that our approach is more robust and consistent respect all the other techniques applied to incremental decision trees.\footnote{This paper is an extension of the work~\cite{de2015splitting} presented at the conference IJCNN 2015.}   

\end{abstract}


\IEEEpeerreviewmaketitle

\section{Introduction}

\label{s:intro}

\IEEEPARstart{S}{tream} mining algorithms are becoming increasingly attractive due to the large number of applications generating large-volume data streams. These include: email, chats, click data, search queries, shopping history, user browsing patterns, financial transactions, electricity consumption, traffic records, telephony data, and so on. In these domains, data are generated sequentially, and scalable predictive analysis methods must be able to process new data in a fully incremental fashion. 
Decision trees classifiers are one of the most widespread non-parametric classification methods. They are fast to evaluate and can naturally deal with mixed-type attributes; moreover, decision surfaces represented by small trees are fairly easy to interpret. Decision trees have been often applied to stream mining tasks ---see, e.g., the survey~\cite{ikonomovska2011learning}. In such settings, the tree growth is motivated by the need of fitting the information brought by the newly observed examples.
Starting from the pioneering work in~\cite{utgoff1989incremental}, the incremental learning of decision trees has received a lot of attention in the past 25 years. Several papers build on the idea of~\cite{MusickCR93}, which advocates the use of measures to evaluate the confidence in choosing a split. These works include Sequential ID3~\cite{Gratch95sequentialinductive}, VFDT~\cite{domingos2000mining}, NIP-H and NIP-N~\cite{jin2003efficient}. Sequential ID3 uses a sequential probability ratio test in order to minimize the number of examples sufficient to choose a good split. This approach guarantees that the tree learned incrementally is close to the one learned via standard batch learning. A similar yet stronger guarantee is achieved by the Hoeffding tree algorithm, which is at the core of the state-of-the-art VFDT system. Alternative approaches, such as NIP-H e NIP-N, use Gaussian approximations instead of Hoeffding bounds in order to compute confidence intervals. Several extensions of VFDT have been proposed, also taking into account non-stationary data sources ---see, e.g., \cite{VFDTc,UFFT,ensambleVFDT,iOVFDT,optionVFDT,hulten2001mining,kirkby2007improving,liu2009ambiguous,gama2011learning,xu2011mining,kosina2012very,salperwyck2013incremental,duda2014novel}. All these methods are based on the classical Hoeffding bound \cite{hoeffding1963probability}: after $m$ independent observations of a random variable taking values in a real interval of size $R$, with probability at least $1-\delta$ the true mean does not differ from the sample mean by more than 
\begin{equation}
\label{eq:hoef}
\ve_{\mathrm{hof}}(m,\delta)=R\sqrt{\frac{1}{2m}\ln\frac{1}{\delta}}~.
\end{equation}
The problem of computing the confidence interval for the split gain estimate can be phrased as follows: we are given a set of unknown numbers $G$ (i.e., the true gains for the available splits) and want to find the largest of them. We do that by designing a sample-based estimator $\Ghat$ of each $G$, and then use an appropriate version of the Hoeffding bound to control the probability that $\bigl|\Ghat-G\bigr| > \ve$ for any given $\ve > 0$. It is easy to see that this allows to pick the best split at any given node: assume that $\Ghat_{F}$ is the highest empirical gain (achieved by the split function $F$) and $\Ghat_{F_2}$ is the second-best (achieved by the split function $F_2$). If $\Ghat_{F}-\Ghat_{F_2} > 2\ve$ then with probability at least $1-\delta$ the split function $F$ is optimal\footnote{In the original work VFDT~\cite{domingos2000mining} $\Ghat_{F}-\Ghat_{F_2}>\ve$ is erroneously used.} ---see Figure~\ref{fig:conf_split}.
\begin{figure}[ht!]
\centering
\begin{tikzpicture}[thick, framed]
    \draw (4.5,0.5)[line width=2pt] -- (6.5,0.5);
    \draw (4.5,0.39)[line width=2pt] -- (4.5,0.61) node[above]{$\Ghat_{F}-\ve$};
    \draw (6.5,0.39)[line width=2pt] -- (6.5,0.61) node[above]{$\Ghat_{F}+\ve$};
    \filldraw[ball color=red!80,shading=ball] (5.5,0.5) circle
        (0.06cm) node[above]{$\Ghat_{F}$};
  
    \draw (1,1.5)[line width=2pt] -- (4,1.5);
      \draw (1,1.39)[line width=2pt] -- (1,1.61) node[above]{$\Ghat_{F_2}-\ve$};
      \draw (4,1.39)[line width=2pt] -- (4,1.61) node[above]{$\Ghat_{F_2}+\ve$};
      \filldraw[ball color=red!80,shading=ball] (2.5,1.5) circle
          (0.06cm) node[above]{$\Ghat_{F_2}$};
    \draw (4,2.1) node[above,xshift=0.1cm]{$ \textsc{Confident Split}$};%
\end{tikzpicture}
\caption{\label{fig:conf_split} The condition $\Ghat_{F}-\Ghat_{F_2} > 2\ve$ guarantees that the confidence intervals for the true gains $G_F$ and $G_{F_2}$ are non-overlapping.}
\end{figure}
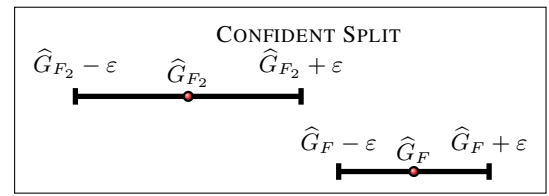
Although all methods in the abovementioned literature use the Hoeffding bound~(\ref{eq:hoef}) to compute the confidence intervals for the splits, we show here that the standard entropy-like criteria require a different approach.

The rest of the paper is organized as follows. Section~\ref{s:rw} discusses related work. In Section~\ref{s:prel} we state the basic decision tree learning concepts and introduce the notation used in the rest of the paper. In Section~\ref{s:main} we derive the new bounds for the splitting criteria. In Section~\ref{s:apps} we apply the confidence bounds to the incremental learning of a decision tree and derive a formal guarantee (Theorem~\ref{th:optim}) on the probability that examples in the stream are classified using suboptimal splits based on any of the three splitting criteria. These theoretical guidelines are empirically tested in Section~\ref{s:exp}, where we show that our more refined bounds deliver better splits that the splits performed by the other techniques. In Section~\ref{s:selective} we develop a selective sampling version of our algorithm using the new confidence bounds.  In this setting, the learner has the possibility of adaptively subsampling the labels of examples in the stream. In other words, everytime a new example arrives, the learner may decide to save the cost of obtaining the label. Note that the learner's predictive performance is nevertheless evaluated on the entire stream, including the examples whose label remains unknown. The approach we propose is based on using the purity of a leaf to decide (at some confidence level) whether its classification is optimal. In Section~\ref{s:SS_exp} we compare our labeling strategy with a recent baseline. Section~\ref{s:concl} concludes the paper.

\section{Related work}
\label{s:rw}
The problem of computing the confidence interval for the splitting gain estimate has two main source of difficulties: First, splitting criteria ---like entropy or Gini index--- are nonlinear functions of the distribution at each node. Hence, appropriate large deviation bounds (such as the McDiarmid bound~\cite{mcdiarmid1989method}) must be used. As the McDiarmid bound controls the \emph{deviations} $\bigl|\Ghat-\E\Ghat\bigr|$, further work is needed to control the \emph{bias} $\bigl|\E\Ghat-G\bigr|$. The first problem was solved (for entropy and Gini) in~\cite{rutkowski2012decision}. The authors used McDiarmid bound to derive estimates of confidence intervals for various split measures. For instance, in a problem with $K$ classes, the bound on the confidence interval for the entropy gain criterion is
\begin{align}
\label{eq:mc}
& \ve_{\mathrm{md}}(m,\delta) = C(K,m)\sqrt{\frac{1}{2m}\ln\frac{1}{\delta}}
\end{align}
where $C(K,m) = 6\bigl(K\log_2 e + \log_2 2m\bigr) + 2\log_2 K$.
The authors proposed to replace Hoeffding bound~(\ref{eq:hoef}) by McDiarmid bound~(\ref{eq:mc}) in the VFDT algorithm and its successors. However, although this allows to control the deviations, the bias of the estimate is ignored.
More recently, in~\cite{duda2014novel} the same authors apply the Hoeffding bound to the entropy splitting criterion, focusing on binary trees and binary classification. They decompose the entropy gain calculation in three components, and apply the Hoeffding bound to each one of them, obtaining a confidence interval estimate for the splitting gains. However, this still ignores the bias of the estimate and, besides, the authors do not consider other types of split functions. 
The work~\cite{matuszyk2013correcting} directly uses the classification error as splitting criterion rather than a concave approximation of it (like the entropy or the Gini index). Though this splitting criterion can be easily analyzed via the Hoeffding bound, its empirical performance is generally not very good ---see Section~\ref{s:main} for more discussion on this.

In this work, we significantly simplify the approach of~\cite{rutkowski2012decision} and extend it to a third splitting criterion. Moreover, we also solve the bias problem, controlling the deviations of $\Ghat$ from the real quantity of interest (i.e., $G$ rather than $\E\Ghat$). Moreover, unlike~\cite{matuszyk2013correcting} and~\cite{duda2014novel}, our bounds apply to the standard splitting criteria. Our analysis shows that the confidence intervals associated with the choice of a suboptimal split not only depend on the number of leaf examples $m$ ---as in bounds~(\ref{eq:hoef}) and~(\ref{eq:mc})--- but also on other problem dependent parameters, as the dimension of the feature space, the depth of the leaves, and the overall number of examples seen so far by the algorithm. As revealed by the experiments in Section~\ref{s:contr_exp}, this allows a more cautious and accurate splitting in complex problems. Furthermore, we point out that our technique can be easily applied to all extensions of VFDT (see Section~\ref{s:intro}) yielding similar improvements, as these extensions all share the same Hoeffding-based confidence analysis as the Hoeffding tree algorithm.

Standard decision tree learning approaches assume that all training instances are labeled and available beforehand. In a true incremental learning setting, instead, in which the classifier is asked to predict the label of each incoming sample, active learning techniques allows us to model the interaction between the learning system and the labeler agent (typically, a human annotator) \cite{settles2012active}. More specifically, such techniques help the learner select a small number of instances for which the annotator should be invoked in order to obtain the true label. The overall goal is to maximize predictive accuracy (measured on the entire set of predicted samples, irrespective to whether the label was queried or not) at any given percentage of queried labels. Recently, in the work~\cite{zliobaite2014active} is showed a general active learning framework which is applied to Hoeffding trees. They present different strategies to annotate the samples considering the output leaf class probabilities. These techniques rely on the class probability estimates at the leaves level without considering confidence-based techniques, that is they do not consider if the estimates are supported by a small or large number of sample labels. In Section~\ref{s:selective} we develop a selective sampling version of our algorithm using the new confidence bounds. The approach we propose is based on using the purity of a leaf to decide (at some confidence level) whether its classification is optimal. Labels of such leaves are then queried at a very small rate, dictated by the confidence level that is let to increase with time and relevant leaf statistics. In the experimental results showed in Section~\ref{s:SS_exp} is clear that our confidence-based strategy is more robust than the strategy described in~\cite{zliobaite2014active} applied to Decision Trees.

\section{Batch Decision Tree Learning}
\label{s:prel}
For simplicity we only consider binary classification problems. The goal is to find a function $f(\bX)$ assigning the correct category $Y=\{0,1\}$ to a new instance $\bX$. We consider binary decision trees based on a class $\scF$ of binary split functions $F : \R^d\to\bool$, that is, the test functions through which the feature space is partitioned\footnote{
Although we only considered binary classification and binary splits, our techniques can be potentially extended to multi-class classification and general splits.
}.
Training examples $(\bX_1,Y_1),(\bX_2,Y_2),\ldots\in\R^d\times\bool$ are i.i.d.\ draws from a fixed but unknown probability distribution. Decision tree classifiers are typically constructed in an incremental way, starting from a tree consisting of a single node. The tree grows through a sequence of splitting operations applied to its leaves. If a split is decided for a leaf $i$, then the leaf is assigned some split function $F \in \scF$ and two nodes $i_0$ and $i_1$ are added to the tree as children of the split node. Examples are recursively routed throught the tree starting from the root as follows: when an example $(\bX_t,Y_t)$ reaches an internal node $i$ with split function $F$, then it is routed to child $i_0$ if $F(\bX_t) = 0$ and to child $i_1$ otherwise.
A decision tree $T$ induces a classifier $f_T : \R^d\to\bool$. The prediction $f_T(\bX)$ of this classifier on an instance $\bX$ is computed by routing the instance $\bX$ through the tree until a leaf is reached. We use $\bX \to i$ to indicate that $\bX$ is routed to the leaf $i$. Then $f_T(\bX)$ is set to the most frequent label $y=\arg \max_{y} \Pr(Y=y|\bX \to i)$ among the labels of all observed examples that reach that leaf.
The goal of the learning process is to control the binary classification risk $\Pr\bigl( f_T(\bX) \neq Y \bigr)$ of $f_T$. For any leaf $i$, let $Y_{|i}=\Pr(Y|\bX \to i)$ be the random variable denoting the label of a random instance $\bX$ given that $\bX \to i$. Let $\scL(T)$ be the leaves of $T$. The risk of $f_T$ can then be upper bounded, with the standard bias-variance decomposition, as follows
\begin{align*}
    \Pr&\bigl( f_T(\bX) \neq Y \bigr)
\\ &\le
	   \overbrace{ \sum_{i \in \scL(T)}  \Pr\bigl(Y \neq y^*_i \mid \bX \to i \bigr) \Pr(\bX \to i) }^{\mbox{bias error}}
\\ &+
   \overbrace{ \sum_{i \in \scL(T)}\Pr\bigl(f_T(\bX) \neq y^*_i \mid \bX \to i\bigr)\Pr(\bX \to i)}^{\mbox{variance error}}
\end{align*}
where $y^*_i = \Ind{\Pr(Y = 1 \mid \bX \to i) \geq \tfrac{1}{2}}$ is the optimal label\footnote{The label assigned by the Bayes optimal  classifier in the leaf partition.} for leaf $i$ and $\Ind\cdot$ is the indicator function of the event at argument. The \emph{variance} terms 
are the easiest to control: $f_T(\bX)$ is determined by the most frequent label of the leaf $i$ such that $\bX \to i$. Hence, conditioned on $\bX \to i$, the event $f_T(\bX) = y^*_i$ holds with high probability whenever the confidence interval for the estimate of $y^*_i$ does not cross the $\tfrac{1}{2}$ boundary\footnote{This confidence interval shrinks relatively fast, as dictated by Hoeffding bound applied to the variable $y^*_i\in\{0,1\}$.}. The bias terms
compute the Bayes error at each leaf. The error vanishes quickly when good splits for expanding the leaves are available. However, due to the large number of available split functions $F$, the confidence intervals for choosing such good splits shrink slower than the confidence interval associated with the bias error. Our Theorem~\ref{th:optim} accurately quantifies the dependence of the split confidence on the various problem parameters.
Let $\Psi(Y)$ be a shorthand for $\min\bigl\{\Pr(Y=0),\Pr(Y=1)\bigr\}$. Every time a leaf $i$ is split using $F$, the term $\Psi(Y_{|i})$ gets replaced by
\begin{align*}
    \Pr&\bigl(F = 0 \mid \bX \to i \bigr)\Psi\bigl(Y_{|i} \mid F=0 \bigr)
\\ &
    + \Pr\bigl(F = 1  \mid \bX \to i \bigr)\Psi\bigl(Y_{|i} \mid F=1 \bigr)
\end{align*}
corresponding to the newly added leaves (here and in what follows, $F$ also stands for the random variable $F(\bX)$). The concavity of $\min$ ensures that no split of a leaf can ever make that sum bigger. Of course, we seek the split maximizing the risk decrease (or ``gain''),
\begin{align*}
    \Psi(Y_{|i}) &- \Pr\bigl(F = 0 \mid \bX \to i \bigr)\Psi\bigl(Y_{|i} \mid F=0\bigr)
\\&-
    \Pr\bigl(F = 1 \mid \bX \to i \bigr)\Psi\bigl(Y_{|i} \mid F=1 \bigr)~.
\end{align*}
In practice, splits are chosen so to approximately maximize a gain functional defined in terms of a concave and symmetric function $\Phi$, which bounds from the above the min function $\Psi$ (used in~\cite{matuszyk2013correcting} as splitting criterion). The curvature of $\Phi$ helps when comparing different splits, as opposed to $\Psi$ which is piecewise linear. Indeed $\Psi$ gives nonzero gain only to splits generating leaves with disagreeing majority labels ---see, e.g., \cite{dietterich1996applying} for a more detailed explanation. Let $Z$ be a Bernoulli random variable with parameter $p$. Three gain functions used in practice are: the scaled binary entropy $\sH(Z) = -\tfrac{p}{2}\ln p - \tfrac{1-p}{2}\ln(1-p)$ used in C4.5; the Gini index $J(Z) = 2p(1-p)$ used in CART, and the function $Q(Z) = \sqrt{p(1-p)}$ introduced by Kearns and Mansour in~\cite{kearns1996boosting} and empirically tested in~\cite{dietterich1996applying}.
Clearly, the binary classification risk can be upper bounded in terms of any upper bound $\Phi$ on the min function $\Psi$,
\begin{align*}
    \Pr\bigl( f_T(\bX) \neq Y \bigr)
&\le
    \!\!\sum_{i \in \scL(T)} \!\!\! \Phi(Y_{|i}) \Pr(\bX \to i)\\
&+
    \!\!\sum_{i \in \scL(T)} \!\!\! \Pr\bigl(f_T(\bX) \neq y^*_i \mid \bX \to i\bigr)\Pr(\bX \to i)~.
\end{align*}
The gain for a split $F$ at node $i$, written in terms of a generic entropy-like function $\Phi$, takes the form
\begin{align}
\nonumber
     G_{i,F}
&=
    \Phi(Y_{|i}) - \Phi\bigl(Y_{|i} \mid F\bigr)\\
\nonumber
&=
    \Phi(Y_{|i}) - \Pr\bigl(F = 0 \mid \bX \to i \bigr)\Phi(Y_{|i} \mid F = 0)
\\ &
\label{eq:gain}
    - \Pr(F = 1 \mid \bX \to i )\Phi(Y_{|i} \mid F = 1)~.
\end{align}
Now, in order to choose splits with a high gain (implying a significant reduction of risk), we must show that $G_{i,F}$ (for the different choices of $\Phi$) can be reliably estimated from the training examples. In this work we focus on estimates for choosing the best split $F$ at any given leaf $i$. Since the term $\Phi(Y_{|i})$ in $G_{i,F}$ is invariant with respect to this choice, we may just ignore it when estimating the gain.\footnote{
Note that, for all functions $\Phi$ considered in this paper, the problem of estimating $\Phi(Y_{|i})$ can be solved by applying essentially the same techniques as the ones we used to estimate $\Phi\bigl(Y_{|i} \mid F\bigr)$.
}

\section{Confidence Bound For Split Functions}
\label{s:main}
In this section we compute estimates $\Dhat_{i\mid F}$ of $\Phi(Y_{|i} \mid F)$ for different choices of $\Phi$, and compute confidence intervals for these estimates. As mentioned in Section~\ref{s:intro}, we actually bound the deviations of $\Dhat_{i\mid F}$ from $\Phi(Y_{|i} \mid F)$, which is the real quantity of interest here. Due to the nonlinearity of $\Phi$, this problem is generally harder than controlling the deviations of $\Dhat_{i\mid F}$ from its expectation $\E\,\Dhat_{i\mid F}$ ---see, e.g., \cite{rutkowski2012decision} for weaker results along these lines.
In the rest of this section, for each node $i$ and split $F$ we write $p_k = \Pr(Y = 1, F = k)$ and $q_k = 1 - p_k$ for $k \in \bool$; moreover, we use $\phat_k$, $\qhat_k$ to denote the empirical estimates of $p_k$, $q_k$.

\subsection{Bound for the entropy}
Let $\Phi(Z)$ be the (scaled) binary entropy $\sH(Z) = -\tfrac{p}{2}\ln p - \tfrac{1-p}{2}\ln(1-p)$ for $Z$ Bernoulli of parameter $p$. In the next result, we decompose the conditional entropy as a difference between entropies of the joint and the marginal distribution. Then, we apply standard results for plug-in estimates of entropy.
\bookbox{
\begin{theorem}
\label{th:entropy}
Pick a node $i$ and route $m$ i.i.d.\ examples $(\bX_t,Y_t)$ to $i$. For any $F\in\scF$, let
\[
    \Dhat_{i\mid F} = \sHhat(Y_{|i},F) - \sHhat(F)
\]
where $\sHhat$ denotes the empirical scaled entropy (i.e., the scaled entropy of the empirical measure defined by the i.i.d.\ sample). Then, for all $\delta > 0$,
\begin{align}
    \nonumber
    \Bigl| \Dhat_{i \mid F} - \sH(Y_{|i} \mid F) \Bigr| \le \ve_{\mathrm{ent}}(m,\delta)
    \\ \text{where} \quad \ve_{\mathrm{ent}}(m,\delta)=(\ln m)\sqrt{\frac{2}{m}\ln\frac{4}{\delta}} + \frac{2}{m}
\end{align}
with probability at least $1-\delta$ over the random draw of the $m$ examples.
\end{theorem}
}
\begin{proof}
In appendix~\ref{pr:entropy}.
\end{proof}

\subsection{Bound for the Gini index}
In the Bernoulli case, the Gini index takes the simple form $J(Z) = 2p(1-p)$ for $Z$ Bernoulli of parameter $p$. First we observe that $J(Y_{|i} \mid F)$ is the sum of harmonic averages, then we use the McDiarmid inequality to control the variance of the plug-in estimate for these averages.
\bookbox{
\begin{theorem}
\label{th:gini}
Pick a node $i$ and route $m$ i.i.d.\ examples $(\bX_t,Y_t)$ to $i$. For any $F\in\scF$, let
\[
    \Dhat_{i \mid F} = \hm(\phat_1,\qhat_1) + \hm(\phat_0,\qhat_0)
\]
where $\hm$ denotes the harmonic mean $\hm(p,q) = \tfrac{2pq}{p+q}$. Then, for all $\delta > 0$
\begin{align}
	\nonumber
    \Bigl| \Dhat_{i \mid F} - J(Y_{|i} \mid F) \Bigr| \le \ve_{\mathrm{Gini}}(m,\delta)
    \\ \text{where} \quad \ve_{\mathrm{Gini}}(m,\delta)=\sqrt{\frac{8}{m}\ln\frac{2}{\delta}} + 4\sqrt{\frac{1}{m}}
\end{align}
with probability at least $1-\delta$ over the random draw of the $m$ examples.
\end{theorem}
}
\begin{proof}
In appendix~\ref{pr:giny}.
\end{proof}

\subsection{Bound for the Kearns-Mansour index}
The third entropy-like function we analyze is $Q(Z) = \sqrt{p(1-p)}$ for $Z$ Bernoulli of parameter $p$. The use of this function was motivated in~\cite{kearns1996boosting} by a theoretical analysis of decision tree learning as a boosting procedure. See also~\cite{takimoto2003top} for a simplified analysis and some extensions.

In this case McDiarmid inequality is not applicable and we control $Q(Y_{|i} \mid F)$ using a direct argument based on classical large deviation results.
\bookbox{
\begin{theorem}
\label{th:squareroot}
Pick a node $i$ and route $m$ i.i.d.\ examples $(\bX_t,Y_t)$ to $i$. 
For any $F\in\scF$, let
\[
    \Dhat_{i \mid F} = \sqrt{\phat_1\qhat_1} + \sqrt{\phat_0\qhat_0}~.
\]
Then, for all $\delta > 0$,
\begin{align}
  \nonumber
    \Bigl| \Dhat_{i \mid F} - Q(Y_{|i} \mid F) \Bigr|
\le \ve_{\mathrm{KM}}(m,\delta)
 \\ \text{where} \quad \ve_{\mathrm{KM}}(m,\delta)=4\sqrt{ \frac{1}{m}\ln\frac{8}{\delta} }
\end{align}
with probability at least $1-\delta$ over the random draw of the $m$ examples.
\end{theorem}
}
\begin{proof}
In appendix~\ref{pr:kearns}.
\end{proof}

\section{Confidence Decision Tree Algorithm}
\label{s:apps}
A setting in which confidence intervals for splits are extremely useful is online or stream-based learning. In this setting, examples are received incrementally, and a confidence interval can be used to decide how much data should be collected at a certain leaf before a good split $F$ can be safely identified. A well-known example of this approach are the so-called Hoeffding trees~\cite{domingos2000mining}. In this section, we show how our confidence interval analysis can be used to extend and refine the current approaches to stream-based decision tree learning. For $t=1,2,\dots$ we assume the training example $(\bX_t,Y_t)$ is received at time $t$.
\begin{algorithm}                   
\caption{C-Tree}          
\label{alg:confTree}                           
\begin{algorithmic}[1]                     
\REQUIRE Threshold $\tau > 0$
    \STATE Build a $1$-node tree $T$
    \FOR{$t=1,2,\dots$}
	    \STATE Route example $(\bX_t,Y_t)$ through $T$ until a leaf $\ell_t$ is reached
	    \IF{$\ell_t$ is not pure}
	    	\STATE Let ${\dt \Fhat = \argmax_{F\in\scF} \Dhat_{\ell_t,F_1} }$ and ${\dt \Fhat_2 = \argmax_{F\in\scF \,:\, F\neq\Fhat} \Dhat_{\ell_t,F_1} }$
	    	\IF{$\Dhat_{\ell_t,\Fhat} \le \Dhat_{\ell_t,\Fhat_2} - 2\ve_t$ \OR $\ve_t \le \tau$}
	    		\STATE Let $F_{\ell_t} = \Fhat$ and expand $\ell_t$ using split $F_{\ell_t}$
	    	\ENDIF
	    \ENDIF
    \ENDFOR    
\end{algorithmic}
\end{algorithm}
C-Tree (Algorithm~1) describes the online decision tree learning approach. A stream of examples is fed to the algorithm, which initially uses a $1$-node decision tree. At time $t$, example $(\bX_t,Y_t)$ is routed to a leaf $\ell_t$. If the leaf is not pure (both positive and negative examples have been routed to $\ell_t$), then the empirically best $\Fhat$ and the second-best $\Fhat_2$ split for $\ell_t$ are computed. If the difference in gain between these two splits exceeds a value $\ve_t$, computed via the confidence interval analysis, then the leaf is split using $\Fhat$. The leaf is also split when $\ve_t$ goes below a ``tie-break'' parameter $\tau$, indicating that the difference between the gains of $\Fhat$ and $\Fhat_2$ is so tiny that waiting for more examples in order to find out the really best split is not worthwhile.
Let $\ve(m,\delta)$ be the size of the confidence interval computed via Theorem~\ref{th:entropy}, \ref{th:gini} or~\ref{th:squareroot}. Fix any node $i$ and let $m_{i,t}$ be the number of examples routed to that node in the first $t-1$ time steps. Clearly, for any $F,F'\in\scF$ with $F \neq F'$, if $\Dhat_{i\mid F} \le \Dhat_{i\mid F'} - 2\ve(m_{i,t},\delta)$ then $G_{i,F} \ge G_{i,F'}$ with probability at least $1-\delta$. Now, since the number of possible binary splits is at most $dm_{i,t}$, if we replace $\delta$ by $\delta/(dm_{i,t})$ the union bound guarantees that a node $i$ is split using the function maximizing the gain.
\begin{lemma}
\label{l:aux}
Assume a leaf $i$ is expanded at time $t$ only if there exists a split $\Fhat$ such that
\[
    \Dhat_{i\mid\Fhat} \le \Dhat_{i\mid F} - 2\ve\bigl(m_{i,t},\delta/(dm_{i,t})\bigl)
\]
for all $F\in\scF$ such that $F \neq \Fhat$. Then,
${\dt
    \Fhat = \argmax_{F\in\scF} G_{i,F}
}$
with probability at least $1-\delta$.
\end{lemma}
The next result provides a bound on the probability that a random example is classified using a suboptimal split. A similar result was proven in~\cite{domingos2000mining} for Hoeffding trees.
\bookbox{
\begin{theorem}
\label{th:optim}
Assume C-Tree (Algorithm~1) is run with
\begin{align}
\label{eq:gen-conf}
    \ve_t = \ve\left(m_{\ell_t,t},\frac{\delta}{(h_t+1)(h_t+2)tdm_{\ell_t,t}}\right)
\end{align}
where $\ve(m,\delta)$ is the size of the confidence interval computed via Theorem~\ref{th:entropy}, \ref{th:gini} or~\ref{th:squareroot} and $h_t$ is depth of $\ell_t$. Then the probability that a random example $\bX$ is routed via a $\tau$-suboptimal split is at most $\delta$.
\end{theorem}
}
\begin{proof}
In appendix~\ref{pr:sub_opt}.  
\end{proof}
\begin{remark}
\label{rm:log_mis}
Theorem~\ref{th:optim} controls the classification of a single random example. 
However, choosing $\delta = \tfrac{1}{t}$, and applying the union bound over the time steps, guarantees that only a logarithmic number of examples in the stream are classified via suboptimal splits.
\end{remark}

\section{Selective Strategy for Decision Tree Learning}
\label{s:selective}
In a truly online learning setting, such as a surveillance system or a medical monitoring application, the classification system is asked to predict the label of each incoming data item. However, training labels can be only obtained through the help of a costly human annotator, who should be inkoved only when the confidence in the classification of the current instance falls below a certain level. Selective sampling allows to model this interaction between the learning system and the labeler agent. In our online decision tree setting, we apply the selective sampling mechanism at the leaf level: when enough examples are routed to leaf $i$ such that the event $f_T(\bX) = y^*_i$ holds with the desired confidence level, the algorithm moderates asking the labels of additional examples that are routed to the leaf. 
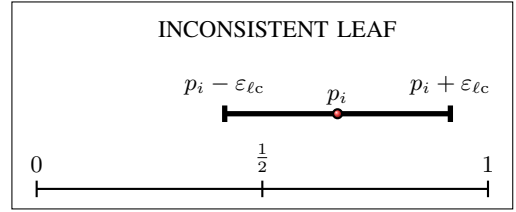
\begin{figure}[ht!]
	\centering
	\begin{tikzpicture}[thick, framed]
	\draw (1,0.5) -- (7,0.5);
	\draw (1,0.39) -- (1,0.61) node[above]{$0$};
	\draw (7,0.39) -- (7,0.61) node[above]{$1$};
	\draw (4,0.39) -- (4,0.61) node[above]{$\frac{1}{2}$};
	
	\draw (3.5,1.5)[line width=2pt] -- (6.5,1.5);
	\draw (3.5,1.39)[line width=2pt] -- (3.5,1.61) node[above]{$p_i -\ve_{\ell\mathrm{c}}$};
	\draw (6.5,1.39)[line width=2pt] -- (6.5,1.61) node[above]{$p_i +\ve_{\ell\mathrm{c}}$};
	\filldraw[ball color=red!80,shading=ball] (5,1.5) circle
	(0.06cm) node[above]{$p_i$};
	\draw (4,2.4) node[above,xshift=0.2cm]{INCONSISTENT LEAF};%
	\end{tikzpicture}
	\caption{\label{fig:conf_leaf} An example of not $\delta$-consistent leaf. The class probability confidence interval overlaps $0.5$. In this case we are not sure, at the desired confidence level, if the prediction made by the leaf is the same of that of the Bayesian Optimal classifier in the correspondent sub-region.}
\end{figure}
\bookbox{
\begin{definition}
Recall that $\ell_t$ is the leaf to which example $(\bX_t,Y_t)$ is routed, and that $m_{\ell_t,t}$ is the number of queried data points routed to leaf $\ell_t$ in the first $t-1$ time steps. Let $\Yhat_{\ell_t,t}$ be the fraction of positive examples among those points, so that $f_T(\bX_t) = \Ind{\Yhat_{\ell_t,t} \ge \tfrac{1}{2}}$. We say that $\ell_t$ is $\delta$-consistent if
\begin{align}
\label{eq:d-cons}
    \bigl|\Yhat_{\ell_t,t} - \tfrac{1}{2}\bigr| > \ve_{\ell\mathrm{c}}(m_{\ell_t,t},t,\delta)
    \\ \text{where} \quad \ve_{\ell\mathrm{c}}(m,t,\delta)=    \sqrt{\frac{1}{2m}\ln\frac{2t}{\delta}}.
\end{align}
\end{definition}
}
Let $p_i = \Pr(Y_{|i} = 1) = \Pr(Y = 1 \mid \bX \to i)$. If $\ell_t$ is $\delta$-consistent but $f_T(\bX_t) \neq y_{\ell_t}^*$, then it must be that $\bigl| \Yhat_{\ell_t,t} - p_{\ell_t} \bigr| \ge \sqrt{\tfrac{1}{2m_{\ell_t,t}}\ln\tfrac{2t}{\delta}}$. Hence, when a leaf becomes $\delta$-consistent we are confident that its classification is optimal at a certain confidence level ---see Figure~\ref{fig:conf_leaf}. On the contrary, when the leaf is not $\delta$-consistent we need to let the leaf seeing more data to reach the consistency. Even in the case of consistency, the leaves need further data in order to discover possible new good splits. In this direction our active method asks a small portion of the input data in order to let exploration even if the leaf is consistent. In this case, as we will describe later with more details, we set the query rate depending on some statistics observed at the leaf level. In particular, we require more exploration for weakly consistent leaf (i.e., a leaf supported by a small portion of data and not pure class distribution) respect to the very confident ones (i.e., a leaf supported by a huge amount of data and with pure class distribution). On the other hand, when the leaf is not consistent, all the labels have to be requested in order to reach the consistency. 

As labels are generally obtained via queries to human annotators, any practical active learning system for streaming settings should impose a bound on the query rate. The active framework we propose is taken from \cite{zliobaite2014active} ---see Algorithm~\ref{alg:online-budget} (ACTIVE setting). Whenever a new sample is presented to the model, the system makes a prediction and then invokes the active learning module in order to determine whether the label should be requested. If this is the case, then a query is issued to the annotator unless the query rate budget is violated. When the label is not requested, the model is not updated. Our labeling strategy is described in Algorithm~\ref{alg:confTree}. In summary, if the incoming sample falls into a not $\delta$-consistent leaf --see Figure~\ref{fig:conf_leaf}-- the annotation is requested. Contrary, on $\delta$-consistent leaf, the annotations are moderated with random sampling to guarantee exploration, that is a controlled growth of the tree. In this case the sampling probability, which gives the priority on the labeling requests,  depends on the budget rate $B$ (more budget more probability), the leaf confidence error $\ve_{\ell\mathrm{c}}$ (less samples support the probability estimates more priority) and the class distribution purity $\bigl|\Yhat_{\ell_t,t} - \tfrac{1}{2}\bigr|$ (distribution toward uniformity more priority). 
%
\bookbox{
\begin{theorem}
The probability that the classification of a $\delta$-consistent leaf is non-optimal is at most $\delta$:
$$ \Pr\bigl( f_T(\bX_t) \neq y_{\ell_t}^*,\; \text{$\ell_t$ is $\delta$-consistent}\bigr)\le\delta$$.
\end{theorem}
}
\begin{proof}
In appendix~\ref{pr:cons_sub_opt}.  
\end{proof}
Similarly to Theorem~\ref{th:optim}, choosing $\delta = \tfrac{1}{t}$ and applying the union bound allows to conclude that at most a logarithmic number of examples in the stream are misclassified by $\delta$-consistent leaves. We use this setting in all the experiments. 
%
\begin{algorithm}[h]                   
\caption{Confidence Tree Strategy}          
\label{alg:confTree}  
\begin{algorithmic}[1]                     
\REQUIRE incoming sample $\bX_t$, decision tree $T$ , budget $B$
\ENSURE \texttt{labeling} $\in\{true, false\}$
    \STATE Route instance $\bX_t$ through $T$ until a leaf $\ell_t$ is reached
   \IF{ ($\ell_t$ is not $\delta$-consistent) \OR 
   ($\ell_t$ is $\delta$-consistent \AND $\texttt{rand}\le \frac{B+\ve_{\ell\mathrm{c}}}{B+\ve_{\ell\mathrm{c}}+\bigl|\Yhat_{\ell_t,t} - \tfrac{1}{2}\bigr|}$ ) }
  		\STATE \textbf{return} $true$  	
  	\ELSE  		
  		\STATE \textbf{return} $false$
  	\ENDIF
\end{algorithmic}
\end{algorithm}
In the work~\cite{zliobaite2014active} are presented different labeling strategies respect to Algorithm~\ref{alg:confTree}. These techniques confide only in the leaf class distribution $Y_{|i}$ and do not take into account the confidence information. The resulting performances are less robust respect that ones achieved by our approaches. This is empirically verified in the experiments showed in Section~\ref{s:SS_exp}.

\section{Full Sampling Experiments}
\label{s:exp}
We ran experiments on synthetic datasets and popular benchmarks, comparing our C-Tree (Algorithm~1) against two baselines: H-Tree (VDFT algorithm~\cite{domingos2000mining}) and CorrH-Tree (the method from~\cite{duda2014novel} using the classification error as splitting criterion). The bounds of~\cite{rutkowski2012decision} are not considered because of their conservativeness. In fact, these bounds generate $1$-node trees in all the experiments, even when the confidence is set to a very low value.

The three methods (ours and the two baselines) share the same core, i.e., the HoeffdingTree (H-Tree) algorithm implemented in MOA\footnote{\url{moa.cms.waikato.ac.nz/}}. In order to implement C-tree and the baseline CorrH-Tree, we directly modified the H-Tree code in MOA.
The grace period parameter\footnote{This is the parameter dictating how many new examples since the last evaluation should be routed to a leaf before revisiting the decision ---see \cite{domingos2000mining}.} was set to $100$. In contrast to the typical experimental settings in the literature, we did not consider the tie-break parameter because in the experiments we observed that it caused the majority of the splits.
Based on Theorem~\ref{th:optim} and Remark~\ref{rm:log_mis}, we used the following version of our confidence bounds $\ve_{\mathrm{KM}}$ and $\ve_{\mathrm{Gini}}$ (the bound for $\ve_{\mathrm{ent}}$ contains an extra $\ln m$ factor),
\begin{align}
\label{eq:heuristic}
\widetilde{\ve}_{\mathrm{KM}}=\widetilde{\ve}_{\mathrm{Gini}}= c \sqrt{ \frac{1}{m}\ln\bigl(m^2 h^2 t d\bigr)}
\end{align}
where the parameter $c$ is used to control the number of splits.

In a preliminary round of experiments, we found that the Gini index delivered a performance comparable to that of entropy and Kearns-Mansour, but ---on average--- produced trees that were more compact for all three algorithms (ours and the two baselines). Hence, we ran all remaining experiments using the Gini index.

In all experiments we measured the \textsl{online performance}. This is the average performance (either accuracy or F-measure) when each new example in the stream is predicted using the tree trained only over the past examples in the stream (``Interleaved Test-Then-Train'' validation in MOA) ---see Algorithm~\ref{alg:online-budget} (FULL setting).

\begin{algorithm}[h]                   
	\caption{Online Stream Validation Protocol}          
	\label{alg:online-budget}  
	\begin{algorithmic}[1]                     
		\REQUIRE labeling budget $B$, Active Strategy, examples stream $(\bx_1,y_1),(\bx_2,y_2),\ldots$
		\STATE Initialize online accuracy $M_0 = 0$
		\FOR{$i=1,2,\ldots$}
		\STATE Receive sample $\bx_i$
		\STATE Predict $\hat{y}_i$
		\STATE Update $M_i = \bigl(1-\tfrac{1}{i}\bigr)M_{i-1} + \tfrac{1}{i}\mathbb{I}\{\hat{y}_i = y_i\}$
		\IF {FULL setting}
		\STATE Receive true label $y_i$  	
		\STATE Update model using new example $(\bx_i,y_i)$
		\ELSE
		\STATE // ACTIVE setting
		\IF {budget $B$ not exceeded}
		\IF {\texttt{Strategy($\bx_i$,\texttt{model})}}
		\STATE Request true label $y_i$
		\STATE Update query rate
		\STATE Update model using new example $(\bx_i,y_i)$
		\ENDIF  	
		\ENDIF
		\ENDIF
		\ENDFOR
	\end{algorithmic}
\end{algorithm}	

\subsection{Controlled Experiments}
\label{s:contr_exp}
%
In order to empirically verify the features of our bounds we performed experiments in a controlled setting. These experiments show how the detailed form of our confidence bound, which ---among other things--- takes into account the number $d$ of attributes and the structure of the tree (through the depth of the leaves to split), allows C-Tree to select splits that are generally better than the splits selected by the baselines. In particular, we generated data streams from a random decision trees with $50$ leaves and observed that C-Tree dominates across the entire range of parameters and ---unlike the other algorithms--- achieves the best accuracy when the number of its leaves is the same as the one of the tree generating the stream.
The random binary trees were generated according to Algorithm~2 with fixed class distributions in each leaf. The random binary trees are constructed through recursive random splits. More precisely, we start at the root with a budget of $n$ leaves. Then we assign to the left and right sub-trees $\floor*{nX}$ and $n-1-\floor*{nX}$ leaves respectively, where $X$ is uniformly distributed in the unit interval. This splitting continues with i.i.d.\ draws of $X$ until the left and right sub-trees are left with one leaf each.
Whenever a split is generated, we assign it a uniformly random attribute and a random threshold value. In the experiment, we generated $1000$ random binary trees with $n=50$ leaves. The random splits are performed choosing among $d=5$ attributes. For simplicity, we only considered numerical attributes and thresholds in the $[0,1]$ interval. A random binary tree is then used to generate a stream as follows: for each leaf of the tree, $10,\!000$ examples are uniformly drawn from the subregion of $[0,1]^{5}$ defined by the leaf, obtaining $500,\!000$ examples. Each of these examples is given label $1$ with probability $0.7$ for a left leaf and with probability $0.3$ for a right leaf.
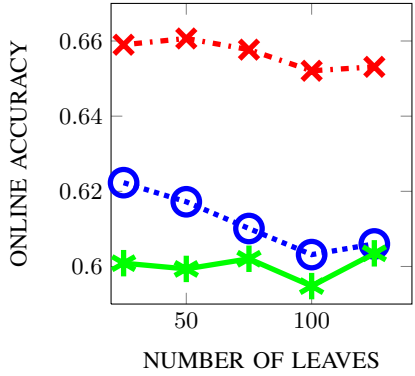
\begin{figure}[t]
\centering
%
%
\begin{tikzpicture}

\begin{axis}[%
width=\figurewidthp,
height=\figureheightp,
scale only axis,
xmin=20,
xmax=140,
xlabel={NUMBER OF LEAVES},
ymin=0.59,
ymax=0.67,
ylabel={ONLINE ACCURACY}
]
\addplot [color=red,dash pattern=on 1pt off 3pt on 3pt off 3pt,line width=2.0pt,mark size=5.0pt,mark=x,mark options={solid},forget plot]
  table[row sep=crcr]{25	0.659004081632654\\
50	0.660671541501977\\
75	0.657700215982721\\
100	0.652077197802198\\
125	0.65312209631728\\
};
\addplot [color=blue,dotted,line width=2.0pt,mark size=5.0pt,mark=o,mark options={solid},forget plot]
  table[row sep=crcr]{25	0.622298306679208\\
50	0.617215674603174\\
75	0.610100867052023\\
100	0.603156485355648\\
125	0.605997260273973\\
};
\addplot [color=green,solid,line width=2.0pt,mark size=5.0pt,mark=asterisk,mark options={solid},forget plot]
  table[row sep=crcr]{25	0.600879834254144\\
50	0.599350410958904\\
75	0.60205038167939\\
100	0.594776262626263\\
125	0.603497333333333\\
};
\end{axis}
\end{tikzpicture}%
\caption{
\label{fig:cmp}
Online accuracy against number of leaves achieved across a grid of $200$ input parameter values on $1000$ synthetic datasets ($c\in(0,2)$ for C-Tree (cross and red line), $\delta \in (0,1)$ for H-Tree (circle and blue line) and CorrH-Tree (star and green line).
}
\end{figure}
In Figure~\ref{fig:cmp} we show online performances averaged over $1000$ streams, each generated using a different random binary tree. In order to span a wide range of tree sizes, we used a grid of $200$ different values for the algorithms' parameters controlling the growth of the trees. Namely, the parameter $\delta$ available in MOA implementation of H-Tree and CorrH-Tree, and the parameter $c$ of~(\ref{eq:heuristic}) for C-Tree (in C-Tree $\delta$ is set to $\tfrac{1}{t}$ according to Remark~\ref{rm:log_mis}). The plots are obtained as follows: for each dataset and algorithm we logged the running average of the online performance --$M_i$ in Algorithm~\ref{alg:online-budget}-- and the total number of leaves in the tree as the stream was being fed to the algorithm.
\begin{algorithm}[htbp]                   
	\caption{RandCBT}         
	\label{alg:random_tree_tree}                           
	\begin{algorithmic}[2]                   
		\REQUIRE tree $T$, total number of leaves \texttt{num-leaves}, number of attributes $d$, leaf class conditional probability $q$
		\ENSURE complete binary tree $T$     
		
		\STATE \texttt{current-node} $i=$ CreateNode()     	
		\IF{\texttt{num-leaves} $== 1$}
		\STATE mark $i$ as leaf 
		\IF{$i$ is a left child}
		\STATE $\Pr(Y=1|\bX \to i)=q$ 
		\ELSE 
		\STATE $\Pr(Y=1|\bX \to i)=1-q$
		\ENDIF
		\ELSE            
		\STATE $x=$ UniformSample$[0,1]$
		\STATE \texttt{left-leaves} $= \max\bigl\{1,\floor*{\mbox{\texttt{num-leaves}}\cdot x}\bigr\}$
		\STATE \texttt{right-leaves} $= \mbox{\texttt{num-leaves}}-\mbox{\texttt{left-leaves}}$
		\STATE $i =$ RandomAttribute$(1,\ldots,d)$
		\STATE $v =$ RandomValueInSubRegion$(i)$
		\STATE add split test $(i,v)$ to $i$
		\STATE \texttt{l-child} $=$ RandCBT$(i,\mbox{\texttt{left-leaves}},d,q)$
		\STATE \texttt{r-child} $=$ RandCBT$(i,\mbox{\texttt{right-leaves}},d,q)$
		\STATE add \texttt{l-child} and \texttt{r-child} as a descendent of $i$        	
		\ENDIF
		\STATE \textbf{return} \texttt{current-node} $i$
	\end{algorithmic}
\end{algorithm}

\subsection{Experiments on real-world data}
\label{s:real_exp}
We constructed ten different streams from each dataset listed below here by taking a random permutation of the examples in it. 
\begin{table}[htp!]
\begin{center}
\begin{tabular}{|c|c|c|c|c|} \hline
Dataset & Dimension & Examples & $|+|$ & $|-|$ \\ \hline
A9A* & 123 & 48842 & 11687 & 37155 \\ \hline
AIRLINES & 7 & 539383 & 240264 & 299119  \\ \hline
COD-RNA* & 8 & 488565 & 162855 & 325710  \\ \hline
COVERTYPE & 54 & 581012 & 283301 & 297711  \\ \hline
ELECTRICITY & 8 & 45312 & 26075 & 19237  \\ \hline
\end{tabular}
\end{center}
\caption{Datasets used for benchmarking.}
\label{tab:data}
\end{table}
A9A, COD-RNA and COVERTYPE are from the LIBSVM binary classification repository\footnote{\url{www.csie.ntu.edu.tw/~cjlin/libsvmtools/datasets/binary.html}}. AIRLINES and ELECTRICITY are from the MOA collection\footnote{\url{moa.cms.waikato.ac.nz/datasets/}}. On the unbalanced datasets (marked with a star in Table~\ref{tab:data}) we used the F-measure on the smallest class to measure performance whereas accuracy was used for the remaining datasets.
The parameters $\delta$ (H-Tree and CorrH-Tree) and $c$ (C-Tree) were individually tuned on each dataset using a grid of $200$ values, hence plots show the online performance of each algorithm when it is close to be optimally tuned.
Even if the datasets are not particularly large, the plots show that trees generated by our algorithm compare favourably with respect to the baselines especially in the first learning phases.
\begin{figure*}[t]
\centering
       \subfigure[A9A]{
       \centering
%
%
\begin{tikzpicture}

\begin{axis}[%
width=\figurewidth,
height=\figureheight,
scale only axis,
xmin=10,
xmax=45,
ymin=0.35,
ymax=0.55
]
\addplot [color=red,dash pattern=on 1pt off 3pt on 3pt off 3pt,line width=2.0pt,mark size=5.0pt,mark=x,mark options={solid},forget plot]
  table[row sep=crcr]{13.3	0.458824583901059\\
23.1	0.510873806858413\\
33.6	0.533396390825486\\
44.1	0.547591780051138\\
};
\addplot [color=blue,dotted,line width=2.0pt,mark size=5.0pt,mark=o,mark options={solid},forget plot]
  table[row sep=crcr]{13.3	0.413886085828309\\
23.1	0.476924233118446\\
33.6	0.521309721653378\\
44.1	0.540170956324409\\
};
\addplot [color=green,solid,line width=2.0pt,mark size=5.0pt,mark=asterisk,mark options={solid},forget plot]
  table[row sep=crcr]{13.3	0.362024053251894\\
23.1	0.412505263801552\\
33.6	0.436660279560818\\
44.1	0.468995432486855\\
};
\end{axis}
\end{tikzpicture}%
        }
        \subfigure[AIRLINES]{
        \centering
%
%
\begin{tikzpicture}

\begin{axis}[%
width=\figurewidth,
height=\figureheight,
scale only axis,
xmin=4,
xmax=10,
ymin=0.621,
ymax=0.631
]
\addplot [color=red,dash pattern=on 1pt off 3pt on 3pt off 3pt,line width=2.0pt,mark size=5.0pt,mark=x,mark options={solid},forget plot]
  table[row sep=crcr]{4.25	0.626656318177265\\
5.95	0.627820459826342\\
7.65	0.628519378047933\\
9.35	0.629552618551688\\
};
\addplot [color=blue,dotted,line width=2.0pt,mark size=5.0pt,mark=o,mark options={solid},forget plot]
  table[row sep=crcr]{4.25	0.621029545504019\\
5.95	0.625387105410639\\
7.65	0.628578582333197\\
9.35	0.6285316788762\\
};
\addplot [color=green,solid,line width=2.0pt,mark size=5.0pt,mark=asterisk,mark options={solid},forget plot]
  table[row sep=crcr]{4.25	0.621869989453983\\
5.95	0.626639149480822\\
7.65	0.62823305576286\\
9.35	0.628890750259176\\
};
\end{axis}
\end{tikzpicture}%
        }
        \subfigure[COD-RNA]{
        \centering
%
%
\begin{tikzpicture}

\begin{axis}[%
width=\figurewidth,
height=\figureheight,
scale only axis,
xmin=20,
xmax=140,
ymin=0.81,
ymax=0.9
]
\addplot [color=red,dash pattern=on 1pt off 3pt on 3pt off 3pt,line width=2.0pt,mark size=5.0pt,mark=x,mark options={solid},forget plot]
  table[row sep=crcr]{38.7	0.85912337360155\\
60.3	0.873884630781767\\
84.6	0.886213146635366\\
135	0.896988006928489\\
};
\addplot [color=blue,dotted,line width=2.0pt,mark size=5.0pt,mark=o,mark options={solid},forget plot]
  table[row sep=crcr]{38.7	0.83775143431066\\
60.3	0.856251168159125\\
84.6	0.870940905683968\\
135	0.894029907222185\\
};
\addplot [color=green,solid,line width=2.0pt,mark size=5.0pt,mark=asterisk,mark options={solid},forget plot]
  table[row sep=crcr]{38.7	0.815160481854328\\
60.3	0.827205762365932\\
84.6	0.826971012872544\\
135	0.829465659326935\\
};
\end{axis}
\end{tikzpicture}%
        }
        \subfigure[COVERTYPE]{
        \centering
%
%
\begin{tikzpicture}

\begin{axis}[%
width=\figurewidth,
height=\figureheight,
scale only axis,
xmin=10,
xmax=90,
ymin=0.7,
ymax=0.76
]
\addplot [color=red,dash pattern=on 1pt off 3pt on 3pt off 3pt,line width=2.0pt,mark size=5.0pt,mark=x,mark options={solid},forget plot]
  table[row sep=crcr]{16.25	0.718718520980687\\
34.45	0.742418892137007\\
55.9	0.748358641440673\\
89.7	0.75904785465026\\
};
\addplot [color=blue,dotted,line width=2.0pt,mark size=5.0pt,mark=o,mark options={solid},forget plot]
  table[row sep=crcr]{16.25	0.708550092760503\\
34.45	0.721226063480858\\
55.9	0.730028937656615\\
89.7	0.749292847550089\\
};
\addplot [color=green,solid,line width=2.0pt,mark size=5.0pt,mark=asterisk,mark options={solid},forget plot]
  table[row sep=crcr]{16.25	0.736082547260424\\
34.45	0.740277748590664\\
55.9	0.744780126445741\\
89.7	0.752193614310229\\
};
\end{axis}
\end{tikzpicture}%
        }
        \subfigure[ELECTRICITY]{
        \centering
%
%
\begin{tikzpicture}

\begin{axis}[%
width=\figurewidth,
height=\figureheight,
scale only axis,
xmin=1,
xmax=6,
ymin=0.68,
ymax=0.75
]
\addplot [color=red,dash pattern=on 1pt off 3pt on 3pt off 3pt,line width=2.0pt,mark size=5.0pt,mark=x,mark options={solid},forget plot]
  table[row sep=crcr]{1	0.723611116607815\\
3	0.737664487815246\\
6	0.741487286257394\\
};
\addplot [color=blue,dotted,line width=2.0pt,mark size=5.0pt,mark=o,mark options={solid},forget plot]
  table[row sep=crcr]{1	0.682319971319971\\
3	0.720663266193623\\
6	0.732408953767738\\
};
\addplot [color=green,solid,line width=2.0pt,mark size=5.0pt,mark=asterisk,mark options={solid},forget plot]
  table[row sep=crcr]{1	0.734542468760408\\
3	0.725810698528939\\
6	0.734927470207188\\
};
\end{axis}
\end{tikzpicture}%
        }\hfill
\caption{
\label{fig:full}
Online performance (accuracy or F-measure) against number of leaves for each bin and dataset achieved by C-Tree (cross and red line), H-Tree (circle and blue line) and CorrH-Tree (star and green line).
} 
\end{figure*}
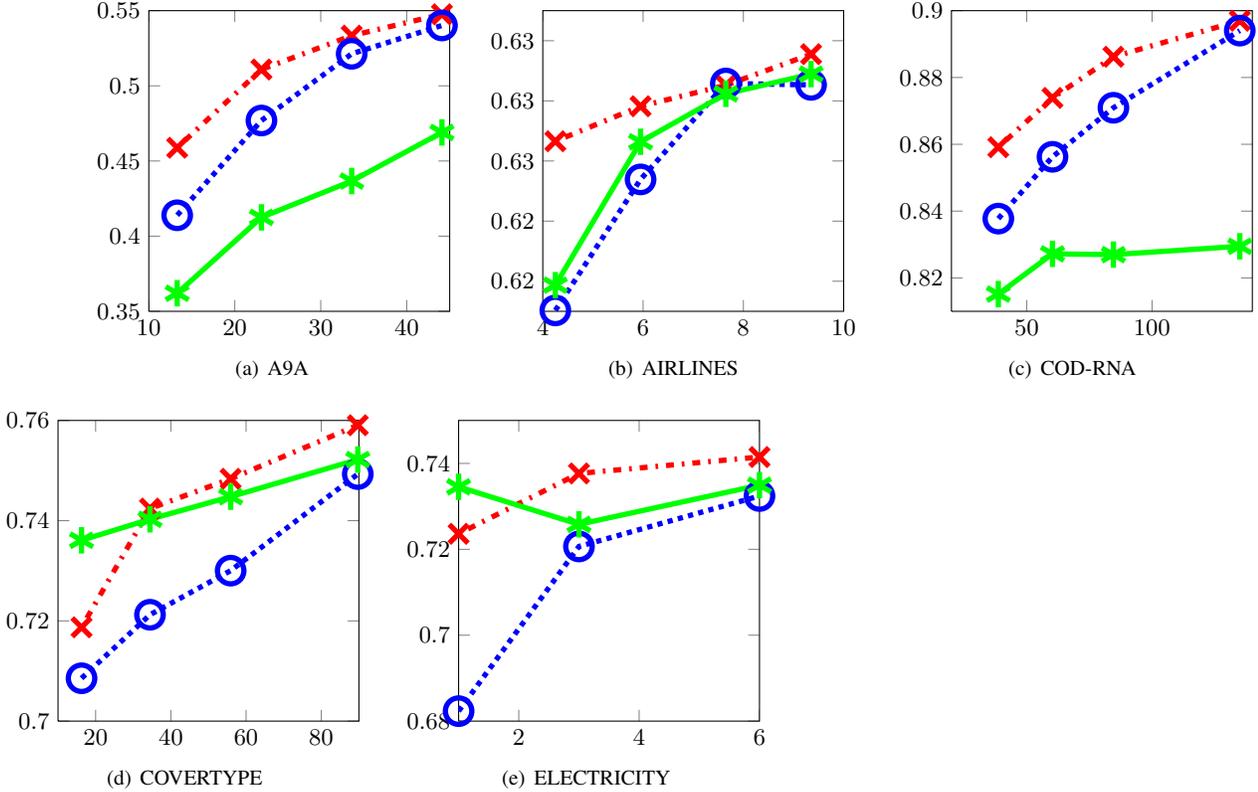
\begin{figure*}[t]
	\centering
	\subfigure[A9A]{
		\centering
%
%
%
\definecolor{mycolor1}{rgb}{0.00000,0.00000,0.17241}%
\definecolor{mycolor2}{rgb}{1.00000,0.10345,0.72414}%
\definecolor{mycolor3}{rgb}{0.00000,0.34483,0.00000}%
\definecolor{mycolor4}{rgb}{0.62069,0.31034,0.27586}%
\begin{tikzpicture}

\begin{axis}[%
width=\figurewidth,
height=\figureheight,
scale only axis,
xmin=0.1,
xmax=0.62,
ymin=0.2176,
ymax=0.42272,
axis x line*=bottom,
axis y line*=left
]
\addplot [color=blue,dashed,line width=1.0pt,mark size=1.8pt,mark=*,mark options={solid,fill=blue},forget plot]
  table[row sep=crcr]{0.1	0.30662\\
0.2	0.3398\\
0.3	0.35722\\
0.4	0.36942\\
0.5	0.39814\\
0.6	0.39564\\
};
\addplot [color=red,dashed,line width=1.0pt,mark size=1.2pt,mark=square*,mark options={solid,fill=red},forget plot]
  table[row sep=crcr]{0.1	0.2276\\
0.2	0.2641\\
0.3	0.2906\\
0.4	0.2896\\
0.5	0.3789\\
0.6	0.3374\\
};
\addplot [color=green,dashed,line width=1.0pt,mark size=1.8pt,mark=+,mark options={solid,fill=green},forget plot]
  table[row sep=crcr]{0.1	0.25292\\
0.2	0.27918\\
0.3	0.3017\\
0.4	0.29366\\
0.5	0.32246\\
0.6	0.3513\\
};
\addplot [color=mycolor1,dashed,line width=1.0pt,mark size=1.2pt,mark=triangle*,mark options={solid,rotate=90,fill=mycolor1},forget plot]
  table[row sep=crcr]{0.1	0.33302\\
0.2	0.33888\\
0.3	0.35342\\
0.4	0.37726\\
0.5	0.41762\\
0.6	0.40806\\
};
\addplot [color=mycolor2,dashed,line width=1.0pt,mark size=1.2pt,mark=triangle*,mark options={solid,fill=mycolor2},forget plot]
  table[row sep=crcr]{0.1	0.25432\\
0.2	0.27436\\
0.3	0.2978\\
0.4	0.32468\\
0.5	0.32006\\
0.6	0.3263\\
};
\addplot [color=mycolor3,dashed,line width=1.0pt,mark size=1.2pt,mark=triangle*,mark options={solid,rotate=270,fill=mycolor3},forget plot]
  table[row sep=crcr]{0.1	0.35058\\
0.2	0.34312\\
0.3	0.35494\\
0.4	0.37366\\
0.5	0.40008\\
0.6	0.4058\\
};
\addplot [color=mycolor4,dashed,line width=1.0pt,mark size=3.0pt,mark=diamond*,mark options={solid,fill=mycolor4},forget plot]
  table[row sep=crcr]{0.1	0.3436\\
0.2	0.34862\\
0.3	0.38158\\
0.4	0.38286\\
0.5	0.37984\\
0.6	0.42172\\
};
\end{axis}
\end{tikzpicture}%
	}
	\subfigure[AIRLINES]{
		\centering
%
%
%
\definecolor{mycolor1}{rgb}{0.00000,0.00000,0.17241}%
\definecolor{mycolor2}{rgb}{1.00000,0.10345,0.72414}%
\definecolor{mycolor3}{rgb}{0.00000,0.34483,0.00000}%
\definecolor{mycolor4}{rgb}{0.62069,0.31034,0.27586}%
\begin{tikzpicture}

\begin{axis}[%
width=\figurewidth,
height=\figureheight,
scale only axis,
xmin=0.1,
xmax=0.62,
ymin=0.6032,
ymax=0.6488,
axis x line*=bottom,
axis y line*=left
]
\addplot [color=blue,dashed,line width=1.0pt,mark size=1.8pt,mark=*,mark options={solid,fill=blue},forget plot]
  table[row sep=crcr]{0.1	0.6156\\
0.2	0.63\\
0.3	0.636\\
0.4	0.6418\\
0.5	0.6437\\
0.6	0.6471\\
};
\addplot [color=red,dashed,line width=1.0pt,mark size=1.2pt,mark=square*,mark options={solid,fill=red},forget plot]
  table[row sep=crcr]{0.1	0.6132\\
0.2	0.6251\\
0.3	0.6332\\
0.4	0.6236\\
0.5	0.6225\\
0.6	0.6243\\
};
\addplot [color=green,dashed,line width=1.0pt,mark size=1.8pt,mark=+,mark options={solid,fill=green},forget plot]
  table[row sep=crcr]{0.1	0.617\\
0.2	0.629\\
0.3	0.635\\
0.4	0.6433\\
0.5	0.6314\\
0.6	0.6382\\
};
\addplot [color=mycolor1,dashed,line width=1.0pt,mark size=1.2pt,mark=triangle*,mark options={solid,rotate=90,fill=mycolor1},forget plot]
  table[row sep=crcr]{0.1	0.6286\\
0.2	0.6332\\
0.3	0.6333\\
0.4	0.6364\\
0.5	0.6411\\
0.6	0.6429\\
};
\addplot [color=mycolor2,dashed,line width=1.0pt,mark size=1.2pt,mark=triangle*,mark options={solid,fill=mycolor2},forget plot]
  table[row sep=crcr]{0.1	0.6152\\
0.2	0.6269\\
0.3	0.6382\\
0.4	0.6399\\
0.5	0.6415\\
0.6	0.6408\\
};
\addplot [color=mycolor3,dashed,line width=1.0pt,mark size=1.2pt,mark=triangle*,mark options={solid,rotate=270,fill=mycolor3},forget plot]
  table[row sep=crcr]{0.1	0.6157\\
0.2	0.6263\\
0.3	0.6366\\
0.4	0.6414\\
0.5	0.6446\\
0.6	0.6463\\
};
\addplot [color=mycolor4,dashed,line width=1.0pt,mark size=3.0pt,mark=diamond*,mark options={solid,fill=mycolor4},forget plot]
  table[row sep=crcr]{0.1	0.6165\\
0.2	0.6333\\
0.3	0.6379\\
0.4	0.6431\\
0.5	0.6449\\
0.6	0.6461\\
};
\end{axis}
\end{tikzpicture}%
	}
	\subfigure[COD-RNA]{
		\centering
%
%
%
\definecolor{mycolor1}{rgb}{0.00000,0.00000,0.17241}%
\definecolor{mycolor2}{rgb}{1.00000,0.10345,0.72414}%
\definecolor{mycolor3}{rgb}{0.00000,0.34483,0.00000}%
\definecolor{mycolor4}{rgb}{0.62069,0.31034,0.27586}%
\begin{tikzpicture}

\begin{axis}[%
width=\figurewidth,
height=\figureheight,
scale only axis,
xmin=0.1,
xmax=0.62,
ymin=0.89764,
ymax=0.92408,
axis x line*=bottom,
axis y line*=left
]
\addplot [color=blue,dashed,line width=1.0pt,mark size=1.8pt,mark=*,mark options={solid,fill=blue},forget plot]
  table[row sep=crcr]{0.1	0.89864\\
0.2	0.90924\\
0.3	0.91398\\
0.4	0.91674\\
0.5	0.91976\\
0.6	0.92054\\
};
\addplot [color=red,dashed,line width=1.0pt,mark size=1.2pt,mark=square*,mark options={solid,fill=red},forget plot]
  table[row sep=crcr]{0.1	0.9007\\
0.2	0.9098\\
0.3	0.9169\\
0.4	0.9162\\
0.5	0.9198\\
0.6	0.9191\\
};
\addplot [color=green,dashed,line width=1.0pt,mark size=1.8pt,mark=+,mark options={solid,fill=green},forget plot]
  table[row sep=crcr]{0.1	0.9\\
0.2	0.91122\\
0.3	0.91492\\
0.4	0.9184\\
0.5	0.92024\\
0.6	0.91968\\
};
\addplot [color=mycolor1,dashed,line width=1.0pt,mark size=1.2pt,mark=triangle*,mark options={solid,rotate=90,fill=mycolor1},forget plot]
  table[row sep=crcr]{0.1	0.90324\\
0.2	0.91318\\
0.3	0.91668\\
0.4	0.91848\\
0.5	0.92048\\
0.6	0.9222\\
};
\addplot [color=mycolor2,dashed,line width=1.0pt,mark size=1.2pt,mark=triangle*,mark options={solid,fill=mycolor2},forget plot]
  table[row sep=crcr]{0.1	0.90166\\
0.2	0.91208\\
0.3	0.91688\\
0.4	0.91936\\
0.5	0.92002\\
0.6	0.92182\\
};
\addplot [color=mycolor3,dashed,line width=1.0pt,mark size=1.2pt,mark=triangle*,mark options={solid,rotate=270,fill=mycolor3},forget plot]
  table[row sep=crcr]{0.1	0.90462\\
0.2	0.91308\\
0.3	0.91648\\
0.4	0.92032\\
0.5	0.92086\\
0.6	0.92308\\
};
\addplot [color=mycolor4,dashed,line width=1.0pt,mark size=3.0pt,mark=diamond*,mark options={solid,fill=mycolor4},forget plot]
  table[row sep=crcr]{0.1	0.90308\\
0.2	0.91422\\
0.3	0.91562\\
0.4	0.91966\\
0.5	0.92282\\
0.6	0.92254\\
};
\end{axis}
\end{tikzpicture}%
	}
	\subfigure[COVERTYPE]{
		\centering
%
%
%
\definecolor{mycolor1}{rgb}{0.00000,0.00000,0.17241}%
\definecolor{mycolor2}{rgb}{1.00000,0.10345,0.72414}%
\definecolor{mycolor3}{rgb}{0.00000,0.34483,0.00000}%
\definecolor{mycolor4}{rgb}{0.62069,0.31034,0.27586}%
\begin{tikzpicture}

\begin{axis}[%
width=\figurewidth,
height=\figureheight,
scale only axis,
xmin=0.1,
xmax=0.62,
ymin=0.5747,
ymax=0.7813,
axis x line*=bottom,
axis y line*=left
]
\addplot [color=blue,dashed,line width=1.0pt,mark size=1.8pt,mark=*,mark options={solid,fill=blue},forget plot]
  table[row sep=crcr]{0.1	0.6446\\
0.2	0.6458\\
0.3	0.7101\\
0.4	0.7313\\
0.5	0.7424\\
0.6	0.7438\\
};
\addplot [color=red,dashed,line width=1.0pt,mark size=1.2pt,mark=square*,mark options={solid,fill=red},forget plot]
  table[row sep=crcr]{0.1	0.6542\\
0.2	0.7107\\
0.3	0.6786\\
0.4	0.7227\\
0.5	0.7296\\
0.6	0.7406\\
};
\addplot [color=green,dashed,line width=1.0pt,mark size=1.8pt,mark=+,mark options={solid,fill=green},forget plot]
  table[row sep=crcr]{0.1	0.6555\\
0.2	0.7028\\
0.3	0.6841\\
0.4	0.74\\
0.5	0.7084\\
0.6	0.7133\\
};
\addplot [color=mycolor1,dashed,line width=1.0pt,mark size=1.2pt,mark=triangle*,mark options={solid,rotate=90,fill=mycolor1},forget plot]
  table[row sep=crcr]{0.1	0.6591\\
0.2	0.668\\
0.3	0.6983\\
0.4	0.7168\\
0.5	0.7107\\
0.6	0.7411\\
};
\addplot [color=mycolor2,dashed,line width=1.0pt,mark size=1.2pt,mark=triangle*,mark options={solid,fill=mycolor2},forget plot]
  table[row sep=crcr]{0.1	0.6527\\
0.2	0.7154\\
0.3	0.673\\
0.4	0.7204\\
0.5	0.7435\\
0.6	0.7411\\
};
\addplot [color=mycolor3,dashed,line width=1.0pt,mark size=1.2pt,mark=triangle*,mark options={solid,rotate=270,fill=mycolor3},forget plot]
  table[row sep=crcr]{0.1	0.6924\\
0.2	0.7442\\
0.3	0.7114\\
0.4	0.7803\\
0.5	0.7346\\
0.6	0.7347\\
};
\addplot [color=mycolor4,dashed,line width=1.0pt,mark size=3.0pt,mark=diamond*,mark options={solid,fill=mycolor4},forget plot]
  table[row sep=crcr]{0.1	0.6865\\
0.2	0.6766\\
0.3	0.7222\\
0.4	0.7533\\
0.5	0.7452\\
0.6	0.7426\\
};
\end{axis}
\end{tikzpicture}%
	}
	\subfigure[ELECTRICITY]{
		\centering
%
%
%
\definecolor{mycolor1}{rgb}{0.00000,0.00000,0.17241}%
\definecolor{mycolor2}{rgb}{1.00000,0.10345,0.72414}%
\definecolor{mycolor3}{rgb}{0.00000,0.34483,0.00000}%
\definecolor{mycolor4}{rgb}{0.62069,0.31034,0.27586}%
\begin{tikzpicture}

\begin{axis}[%
width=\figurewidth,
height=\figureheight,
scale only axis,
xmin=0.1,
xmax=0.62,
ymin=0.7313,
ymax=0.8151,
axis x line*=bottom,
axis y line*=left
]
\addplot [color=blue,dashed,line width=1.0pt,mark size=1.8pt,mark=*,mark options={solid,fill=blue},forget plot]
  table[row sep=crcr]{0.1	0.7574\\
0.2	0.761\\
0.3	0.7635\\
0.4	0.7905\\
0.5	0.7983\\
0.6	0.7927\\
};
\addplot [color=red,dashed,line width=1.0pt,mark size=1.2pt,mark=square*,mark options={solid,fill=red},forget plot]
  table[row sep=crcr]{0.1	0.7513\\
0.2	0.762\\
0.3	0.7749\\
0.4	0.7758\\
0.5	0.7984\\
0.6	0.8004\\
};
\addplot [color=green,dashed,line width=1.0pt,mark size=1.8pt,mark=+,mark options={solid,fill=green},forget plot]
  table[row sep=crcr]{0.1	0.7609\\
0.2	0.7653\\
0.3	0.7723\\
0.4	0.7969\\
0.5	0.8141\\
0.6	0.7942\\
};
\addplot [color=mycolor1,dashed,line width=1.0pt,mark size=1.2pt,mark=triangle*,mark options={solid,rotate=90,fill=mycolor1},forget plot]
  table[row sep=crcr]{0.1	0.7592\\
0.2	0.7557\\
0.3	0.7665\\
0.4	0.7934\\
0.5	0.8045\\
0.6	0.7935\\
};
\addplot [color=mycolor2,dashed,line width=1.0pt,mark size=1.2pt,mark=triangle*,mark options={solid,fill=mycolor2},forget plot]
  table[row sep=crcr]{0.1	0.7517\\
0.2	0.7747\\
0.3	0.7861\\
0.4	0.7853\\
0.5	0.804\\
0.6	0.8091\\
};
\addplot [color=mycolor3,dashed,line width=1.0pt,mark size=1.2pt,mark=triangle*,mark options={solid,rotate=270,fill=mycolor3},forget plot]
  table[row sep=crcr]{0.1	0.773\\
0.2	0.7904\\
0.3	0.7941\\
0.4	0.7934\\
0.5	0.8125\\
0.6	0.8071\\
};
\addplot [color=mycolor4,dashed,line width=1.0pt,mark size=3.0pt,mark=diamond*,mark options={solid,fill=mycolor4},forget plot]
  table[row sep=crcr]{0.1	0.7691\\
0.2	0.7939\\
0.3	0.7864\\
0.4	0.8069\\
0.5	0.8052\\
0.6	0.8125\\
};
\end{axis}
\end{tikzpicture}%
	}
	\subfigure[LEGEND]{
		\centering
		\includegraphics[width=4cm,height=4cm]{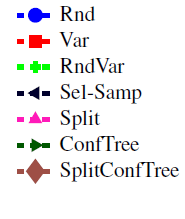}
	}
	\caption{
		\label{fig:ss}
		Online performance (accuracy or F-measure) against the labeling budget.
	} 
\end{figure*}		 
	 
\section{Selective Sampling Experiments}		 
\label{s:SS_exp}

The validation protocol of the active strategies experiment is described in Algorithm~\ref{alg:online-budget} (ACTIVE setting), where we used a different labeling strategy  ($\texttt{line}$ 12) for each compared approach. The labeled instances are stored and used to update the model. The query rate is upper bounded by an input budget parameter $B\in[0,1]$. In these experiments, we calculated the query rate as the fraction of instances for which a label was requested among the ones observed so far ---see~\cite{zliobaite2014active}. We compared our active strategy of Section~\ref{s:selective} against a baseline of five
techniques proposed in~\cite{zliobaite2014active}:

\smallskip

\noindent\textbf{\texttt{Rnd.}} The Random Strategy (Algorithm~\ref{alg:rand_strat}) is a naive method that queries the labels of incoming instances with probability equal to the query rate budget $B$ without considering the actual incoming instance $\bX_t$.

\begin{algorithm}[h]                   
\caption{Random Strategy}          
\label{alg:rand_strat}  
\begin{algorithmic}[1]                     
\REQUIRE labeling budget $B$
\ENSURE \texttt{labeling} $\in\{true, false\}$ indicates whether to
request the true label $y_t$ for $\bx_t$  
  	\STATE Generate a uniform random variable $\texttt{rand}\in[0,1]$;
  	\STATE \textbf{return} $\Ind{\texttt{rand}\le B}$ //where $\Ind\cdot$ is the Indicator Fun.
\end{algorithmic}
\end{algorithm}

\smallskip

\noindent\textbf{\texttt{VarUn.}} We used Variable Uncertainty Strategy described in Algorithm~\ref{alg:vu_strat} to decide for which instances manual annotation is requested. The confidence threshold $\Theta$ which determines requests for new labels, is continuously updated. If the classifier's confidence was above the current threshold $\Theta$ over the time interval associated with the last instance, the latter is increased by a fraction $s$ in order to query only the most uncertain instances. In the opposite case the threshold is reduced, with the goal of acquiring more labels in regions where the estimator is less confident. As explained in~\cite{zliobaite2014active} the parameter $s$ can be easily set to a default value $0.01$. We performed all the experiments with this setting.

\begin{algorithm}[h]                   
	\caption{Variable Uncertainty Strategy}          
	\label{alg:vu_strat}  
	\begin{algorithmic}[1]                     
		\REQUIRE incoming sample $\bx_t$, decision tree $T$, threshold adjustment step $s\in(0, 1]$ 
		\ENSURE \texttt{labeling} $\in\{true, false\}$
		\STATE \textbf{Initialize:} confidence threshold $\Theta=1$ and store
		the latest value during operation
		\STATE Route instance $\bX_t$ through $T$ until a leaf $\ell_t$ is reached 
		\IF {$\max\bigl\{p_{\ell_t},1-p_{\ell_t}\bigr\}<\Theta$}  
		\STATE // confidence below the threshold 
		\STATE decrease the confidence threshold $\Theta = (1-s)\Theta$
		\STATE \textbf{return} $true$
		\ELSE
		\STATE // confidence above the threshold
		\STATE increase the confidence threshold $\Theta = (1 + s)\Theta$
		\STATE \textbf{return} $false$
		\ENDIF
	\end{algorithmic}
\end{algorithm}

\smallskip

\noindent\textbf{\texttt{RndVar.}} This method is essentially the same as \texttt{VarUn} described above. \texttt{VarUn} always labels the instances that are close to the decision boundary. However, in data
streams changes may happen anywhere in the instance space. When concept drift~\cite{widmer1996learning} happens in labels, the classifier will not notice it without the true labels. In order
not to miss concept drift, this technique randomize the labeling threshold
by multiplying by a normally distributed random variable that
follows $\mathcal{N}(1, \delta=1)$. This way, the instances that are
close to the decision boundary are labeled more often, but occasionally
also some distant instances are annotated --see~\cite{zliobaite2014active} for more details.

\smallskip

\noindent\textbf{\texttt{Sel-Samp.}} The Selective Sampling method
is based on~\cite{cesa2006worst}, and uses a variable labeling threshold
$\frac{B}{B+\bigl|\Yhat_{\ell_t,t} - \tfrac{1}{2}\bigr|}$ similar to our random sampling mechanism for $\delta$-consistent leaves. The threshold is based on certainty expectations, and the labels are queried at random.

\smallskip

\noindent\textbf{\texttt{Split.}} Many adaptive learning methods use change-detection
mechanisms that monitor streaming error. Change detectors
(e.g., DDM~\cite{gama2004learning}) are built with an implicit assumption that the
errors are distributed uniformly over time unless a change has happened. The uncertainty strategy asks for labels based on a prediction. Since the predictive model adapts over time, the stream of labeled data is not distributed identically to the stream of unlabeled data. Thus, change detectors may have problems to distinguish a change in distribution due to active labeling from a change in distribution due to concept drift. To overcome that problem, this Split Strategy (Algorithm~\ref{alg:split_strat})
splits a stream at random into two streams. One of the new
streams is labeled according to the  Variable Uncertainty Strategy, while
the other is labeled according to the Random Strategy. Both
streams are used to train a classifier. But only the random
stream is used for change detection. In the experiments we set the parameter $\nu=.2$.  
\begin{algorithm}[h]                   
	\caption{Split Strategy}          
	\label{alg:split_strat}  
	\begin{algorithmic}[1]                     
		\REQUIRE incoming sample $\bx_t$, decision tree $T$, threshold adjustment step $s\in(0, 1]$, proportion of
		random labeling $\nu\in(0,1]$  
		\ENSURE \texttt{labeling} $\in\{true, false\}$ 
		\STATE \textbf{Initialize:} confidence threshold $\Theta=1$ and store
			the latest value during operation
		\STATE Route instance $\bX_t$ through $T$ until a leaf $\ell_t$ is reached
		\STATE Generate a uniform random variable $\texttt{rand}\in[0,1]$;
		\IF {$\texttt{rand}\le\nu$}
			\STATE \texttt{Change Detection Method()}
			\STATE \textbf{return} \texttt{Rnd}($B$)
		\ELSE
			\STATE \textbf{return} \texttt{VarUn}($\bx_i$,$T$,$s$)
		\ENDIF		
	\end{algorithmic}
\end{algorithm}

\smallskip

We compared against the above baseline our Confidence Tree Strategy (Algorithm~\ref{alg:confTree}) that we define as \texttt{ConfTree}. We also coupled the Split Strategy (Algorithm~\ref{alg:split_strat}) with our approach substituting the \texttt{VarUn} procedure with \texttt{ConfTree} method, we define this approach \texttt{SplitConfTree}. All our experiments was performed using the MOA data
stream software suite. We added change detection to the base classifier to improve its performance. We chose DDM~\cite{gama2004learning} as in~\cite{zliobaite2014active}.
All the tested ACTIVE strategies used C-Tree (Algorithm~1) as the base learner with the same parameters setting of Section~\ref{s:exp}. The algorithms  had to predict the label of each new incoming sample. After each prediction, if the active learning system requested the true label, the sample together with its label were fed to it as a new training example ---see Algorithm~\ref{alg:online-budget}. We ran all the competing algorithms with the same range  $B\in\{.1,.2,.3,.4,.5,.6\}$\footnote{Note that the budget is only an upper limit to the actual query rate -- algorithms generally ask for a smaller number of annotations.} of budget values, and plotted in Figure~\ref{fig:ss} the resulting online accuracy as a function of the labeling budget. As for the FULL sampling experiments we deactivate the tie-break mechanism and we tuned the tree learning parameters. The datasets were not randomized to keep unchanged the data order allowing to analyze the methods behaviour in case of original concept drift. Although the performances slightly oscillate probably due to the splitting mechanism that could be sensible to the sub-sampled data class distributions, all the plots exhibit rising accuracy trend as the budget increases, which is to be expected. If there was no upward tendency, then we would conclude that we have excess data and we should be able to achieve a sufficient accuracy by a simple random sub-sampling. The plots clearly show how our approaches remain consistent in the various dataset scenarios in respect to the baseline methods which alternate the best performers ordering depending on the dataset. This reflects that considering more statistical information, as our methods do, permits to achieve more  robust performances. Furthermore, if we compare the performances of Figure~\ref{fig:full} and Figure~\ref{fig:ss}, we can argue that a small fraction of the available labels is sufficient to achieve the performance close to that of the full sampling algorithms.

\section{Conclusions and future works}
\label{s:concl}
The goal of this work was to provide a more rigorous statistical analysis of confidence intervals for splitting leaves in decision trees. Our confidence bounds take into account all the relevant variables of the problem. This improved analysis is reflected in the predictive ability of the learned decision trees, as we show in the experiments. It is important to note that the proposed bounds can be easily applied to the many proposed variants of VFDT. Furthermore, we showed how these bounds can be used to save labels in online decision
tree learning through a selective sampling technique. Both FULL and ACTIVE
applications are supported by theoretical and empirical results. Our confidence analysis applies to i.i.d. streams. Extending
our results to more general processes remains an open problem.
The selective sampling results raise interesting issues concerning
the interplay between nonparametric learning models (such
as decision trees) and subsampling techniques. For example,
there are no theoretical bounds showing the extent to which
labels can be saved without significantly hurting performance.

\appendices

\section{Proof Theorem 1}
\label{pr:entropy}

Let $H$ be the standard (unscaled) entropy. Using the standard identity $H\bigl(Y_{|i} \mid F\bigr) = H\bigl(Y_{|i},F\bigr) - H(F)$, we have $\Dhat_{i \mid F} = \sHhat\bigl(Y_{|i},F\bigr) - \sHhat(F)$. We now use part~(iii) of the remark following \cite[Corollary~1]{antos2001convergence}, we have that
\begin{align*}
    \Bigl|\sHhat(F) - \E\,\sHhat(F)\Bigr| &\le \frac{\ln m}{2}\sqrt{\frac{2}{m}\ln\frac{4}{\delta}}
\\
    \Bigl|\sHhat(Y_{|i},F) - \E\,\sHhat(Y_{|i},F)\Bigr| &\le \frac{\ln m}{2}\sqrt{\frac{2}{m}\ln\frac{4}{\delta}}
\end{align*}
simultaneously hold with probability at least $1-\delta$. These bounds hold irrespective to the size of the sets in which $Y_{|i}$ and $F$ take their values.

Next, we apply~\cite[Proposition~1]{paninski2003estimation}, which states that
\[
    -\ln\left(1 + \frac{N-1}{m}\right) \le \E\,\Hhat(Z) - H(Z) \le 0
\]
for any random variable $Z$ which takes $N$ distinct values. In our case, $N=2$ for $Z = F$ and $N=4$ for $Z = \bigl(Y_{|i},F\bigr)$. Hence, using $-a \le -\ln(1+a)$ for all $a$, we get
\begin{align*}
    \Bigl|\sH(F) - \E\,\sHhat(F)\Bigr| &\le \frac{1}{2m}
\\
    \Bigl|\sH(Y_{|i},F) - \E\,\sHhat(Y_{|i},F)\Bigr| &\le \frac{3}{2m}~.
\end{align*}
Putting everything together gives the desired result.

\section{Proof Theorem 2}
\label{pr:giny}
\begin{lemma}[McDiarmid's inequality]
\label{lm:mcdiarmid}
Let $G$ be a real function of $m$ independent random variables $X_1,\ldots,X_m$ such that
\begin{equation}
\label{eq:perturb}
    \Bigl| G(x_1,\dots,x_i,\dots,x_m) - G(x_1,\dots,x_i',\dots,x_m)\Bigr| \le c
\end{equation}
for some constant $c\in\R$ and for all realizations $x_1,\dots,x_i,x_i',\dots,x_m$. Then 
\[
    \Pr\bigl( \big|G - \E\,G\big| \ge \epsilon\bigr)
\le
    2\exp\left(\frac{-2\epsilon^2}{m\,c^2}\right)~.
\]
If we set the right-hand side equal to $\delta$, then
\[
    \bigl|G - \E\,G\bigr| \le c \sqrt{ \frac{m}{2}\ln\frac{2}{\delta} }
\]
is true with probability at least $1-\delta$.
\end{lemma}
Note the following fact
\begin{align}
\nonumber
    J&(Y_{|i} \mid F)
\\ &=
\nonumber
    \Pr(F = 1)\,2\,\frac{\Pr(Y_{|i} = 1, F = 1)}{\Pr(F = 1)}\,\frac{\Pr(Y_{|i} = 0, F = 1)}{\Pr(F = 1)}
\\ &
\nonumber
    + \Pr(F = 0)\,2\,\frac{\Pr(Y_{|i} = 1, F = 0)}{\Pr(F = 0)}\,\frac{\Pr(Y_{|i} = 0, F = 0)}{\Pr(F = 0)}
\\ &=
\nonumber
    2\,\frac{\Pr(Y_{|i} = 1, F = 1)\,\Pr(Y_{|i} = 0, F = 1)}{\Pr(F = 1)}
\\ &
\nonumber
    + 2\,\frac{\Pr(Y_{|i} = 1, F = 0)\,\Pr(Y_{|i} = 0, F = 0)}{\Pr(F = 0)}
\\ &=
\label{eq:mylast}
    \hm(p_1,q_1) + \hm(p_0,q_0)~.
\end{align}
In view of applying McDiarmid inequality, let $\phat_k = \tfrac{r}{m}$ and $\qhat_k = \tfrac{s}{m}$. We can write the left-hand side of condition~(\ref{eq:perturb}) in Lemma~\ref{lm:mcdiarmid} for each term of~(\ref{eq:mylast}) as
\begin{align*}
   \frac{2}{m}\left| \frac{rs}{r+s} - \frac{r's'}{r'+s'} \right|
\end{align*}
where $r,s = 1,\dots,m$ and $r',s'$ may take the following forms: $(r+1,s-1)$ ---when a label of an example in the current leaf is flipped, $(r+1,s)$ ---when an example is moved from the sibling leaf to the current leaf, and $(r-1,s)$ ---when an example is moved from the current leaf to the sibling leaf. Since the harmonic mean is symmetric in $r$ and $s$, we can ignore the cases $(r-1,s+1)$, $(r,s+1)$, and $(r,s-1)$. A tedious but simple calculation shows that
\begin{align*}
    \left| \frac{rs}{r+s} - \frac{r's'}{r'+s'} \right| \le 1~.
\end{align*}
Therefore, we may apply Lemma~\ref{lm:mcdiarmid} with $c = \tfrac{4}{m}$ and obtain that
\begin{align}
\label{eq:gini-var}
    \Bigl|\Dhat_{i,F} - \E\,\Dhat_{i,F}\Bigr| \le \sqrt{\frac{8}{m}\ln\frac{2}{\delta}}
\end{align}
holds with probability at least $1-\delta$.

Next, we control the bias of $\hm\bigl(\phat_k,\qhat_k\bigr)$ as follows,
\begin{align}
\nonumber
  &  0
\le
    \hm(p_k,q_k) - \E\Bigl[\hm\bigl(\phat_k,\qhat_k\bigr)\Bigr]
\\&=
    2\frac{p_kq_K}{p_k+q_k} - 2\E\left[\frac{\phat_k\qhat_k}{\phat_k+\qhat_k}\right]
\\ &=
\nonumber
    2\E\left[\frac{p_k\phat_k(q_k-\qhat_k) + q_k\qhat_k(p_k-\phat_k)}{(p_k+q_k)\bigl(\phat_k+\qhat_k\bigr)}\right]
\\ &\le
\nonumber
    2\E\left|q_k-\qhat_k\right| + 2\E\bigl|p_k-\phat_k\bigr|
\\ &\le
\nonumber
    2\sqrt{\E\left[\bigl(q_k-\qhat_k\bigr)^2\right]} + 2\sqrt{\E\left[\bigl(p_k-\phat_k\bigr)^2\right]}
\\ &\le
\label{eq:gini-bias}
    \frac{2}{\sqrt{m}}
\end{align}
where the first inequality is due to the concavity of $\hm$. Combining~(\ref{eq:gini-var}) and~(\ref{eq:gini-bias}) concludes the proof.

\section{Proof Theorem 3}
\label{pr:kearns}
\begin{lemma}
\label{lm:okamoto}
Let $B$ binomially distributed with parameters $(m,p)$. Then
\[
    \left|\sqrt{\frac{B}{m}} - \sqrt{p}\right| \le \sqrt{ \frac{1}{m}\ln\frac{2}{\delta} }
\]
is true with probability at least $1-\delta$.
\end{lemma}
\begin{proof}
The result is an immediate consequence of the bounds in~\cite{okamoto1959some} ---see also~\cite[Exercise~2.13]{BLM13}. In particular,
\[
    \Pr\left( \left|\frac{B}{m} - p\right| \ge \ve \right) \le e^{-mD(p+\ve\|p)} + e^{-mD(p-\ve\|p)}
\]
where $D(q\|p) = q\ln\tfrac{q}{p} + (1-q)\ln\tfrac{1-q}{1-p}$ is the KL divergence, and
\begin{align*}
    D(p+\ve\|p) &\ge \left(\sqrt{p+\ve} - \sqrt{p}\right)^2
\\
    D(p-\ve\|p) &\ge 2\left(\sqrt{p-\ve} - \sqrt{p}\right)^2~.
\end{align*}
Simple algrebraic manipulation concludes the proof.
\end{proof}
Similarly to the proof of Theorem~\ref{th:gini}, note that
\begin{align*}
    Q&(Y_{|i} \mid F)
\\ &=
    \Pr(F = 1)\sqrt{\frac{ \Pr(Y_{|i} = 1, F = 1)}{\Pr(F = 1)}\,\frac{\Pr(Y_{|i} = 0, F = 1)}{\Pr(F = 1)} }
\\ &
    + \Pr(F = 0)\sqrt{ \frac{\Pr(Y_{|i} = 1, F = 0)}{\Pr(F = 0)}\,\frac{\Pr(Y_{|i} = 0, F = 0)}{\Pr(F = 0)} }
\\ &=
    \sqrt{ \Pr(Y_{|i} = 1, F = 1)\,\Pr(Y_{|i} = 0, F = 1) }
\\ &
    + \sqrt{ \Pr(Y_{|i} = 1, F = 0)\,\Pr(Y_{|i} = 0, F = 0) }
\\ &=
    \sqrt{p_1q_1} + \sqrt{p_0q_0} 
\end{align*}
Then
\begin{align*}
   & \Bigl| \sqrt{\phat_1\qhat_1} + \sqrt{\phat_0\qhat_0} - \sqrt{p_1q_1} - \sqrt{p_0q_0} \Bigr|
\\ &\le
    \Bigl| \sqrt{\phat_1\qhat_1} - \sqrt{p_1q_1} \Bigr| + \Bigl| \sqrt{\phat_0\qhat_0} - \sqrt{p_0q_0} \Bigr|
\le
    4\ve
\end{align*}
whenever $\bigl|\sqrt{\phat_k} - \sqrt{p_k}\bigr| \le \ve$ and $\bigl|\sqrt{\qhat_k} - \sqrt{q_k}\bigr| \le \ve$ for $k\in\bool$. Using the union bound and Lemma~\ref{lm:okamoto} we immediately get
\[
    \Bigl| \sqrt{\phat_1\qhat_1} + \sqrt{\phat_0\qhat_0} - \sqrt{p_1q_1} - \sqrt{p_0q_0} \Bigr|
\le
    4\sqrt{ \frac{1}{m}\ln\frac{8}{\delta} }
\]
thus concluding the proof.

\section{Proof Theorem 4}
\label{pr:sub_opt}

Fix some arbitrary tree $T$ of depth $H$ and let $\scD_h$ be the set of internal nodes at depth $h$. Clearly,
\[
    \sum_{i\in\scD_h} \Pr(\bX\to i) \le 1~.
\]
Now, for any internal node $i$ of $T$, let $F_i$ be the split used at that node. We have
\begin{align*}
    \Pr&\bigl(\text{$\bX$ routed via a $\tau$-suboptimal split}\bigr)
\\&=
     \Pr\bigl(\exists i \; : \; \bX\to i,\; G_{i,F_i} + \tau < \max_{F\in\scF} G_{i,F}\bigr)
\\ &\le
    \sum_i \Pr\bigl(G_{i,F_i} + \tau < \max_{F\in\scF} G_{i,F} \mid \bX\to i \bigr)\Pr(\bX\to i)
\\ &=
    \sum_{h=0}^H \sum_{i\in\scD_h} \Pr\bigl(G_{i,F_i} + \tau < \max_{F\in\scF} G_{i,F} \mid \bX\to i \bigr)\Pr(\bX\to i)
\\ &\le
    \sum_{h=0}^H \sum_{i\in\scD_h} \Pr\Biggl( \Dhat_{i\mid\Fhat} > \Dhat_{i\mid F}
\\ & \qquad
    - 2\ve\left(m_{i,t},\frac{\delta/\bigl[(h+1)(h+2)\bigr]}{tdm_{i,t}}\right) \Bigg|\, \bX\to i \Biggr) \Pr(\bX\to i) 
\\ &\le
    \sum_{h=0}^H \sum_{i\in\scD_h} \sum_{s=0}^{t-1} \Pr\Biggl( \Dhat_{i\mid\Fhat} > \Dhat_{i\mid F}
\\ & \qquad
    - 2\ve\left(s,\frac{\delta/\bigl[(h+1)(h+2)\bigr]}{tds}\right) \Bigg|\, \bX\to i \Biggr) \Pr(\bX\to i)
\\ &\le
    \sum_{h=0}^H \sum_{i\in\scD_h} \sum_{s=0}^{t-1} \frac{\delta}{(h+1)(h+2)t} \Pr(\bX\to i)
\quad
        \text{(by Lemma~\ref{l:aux})}
\\ &\le
    \sum_{h=0}^H \sum_{s=0}^{t-1}  \frac{\delta}{(h+1)(h+2)t} \le \delta~.
\end{align*}

\section{Proof Theorem 5}
\label{pr:cons_sub_opt}

\begin{align*}
   & \Pr\bigl( f_T(\bX_t) \neq y_{\ell_t}^*,\; \text{$\ell_t$ is $\delta$-consistent}\bigr)
\\&\le
    \Pr\left( \bigl| \Yhat_{\ell_t,t} - p_{\ell_t} \bigr| \ge \sqrt{\frac{1}{2m_{\ell_t,t}}\ln\frac{2t}{\delta}} \right)
\\ &\le
    \sum_{s=0}^{t-1} \Pr\left( \bigl| \Yhat_{\ell_t,t} - p_{\ell_t} \bigr| \ge \sqrt{\frac{1}{2s}\ln\frac{2t}{\delta}} \right)
\\ &\le
    \sum_{s=0}^{t-1} \sum_{i \in \scL(T)} \Pr\left( \left. \bigl| \Yhat_{i,t} - p_i \bigr| \ge \sqrt{\frac{1}{2s}\ln\frac{2t}{\delta}} \,\right| \bX_t \to i \right) \\ & \qquad \Pr(\bX_t \to i)
\\ &\le
    \sum_{s=0}^{t-1} \sum_{i \in \scL(T)} \frac{\delta}{t} \Pr(\bX_t \to i) = \delta
\end{align*}
where we used the standard Chernoff bound in the last step.


\bibliographystyle{plain}

\begin{IEEEbiography}[{\includegraphics[width=1in,height=1.2in,clip,keepaspectratio]{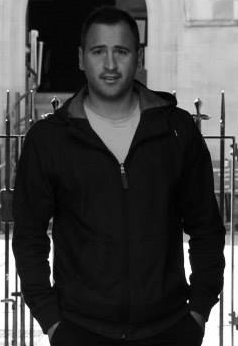}}]{Rocco De Rosa}	received the Ph.D. degree from
University of Milan, Milan, Italy, in 2014.
	He is currently a Researcher in Machine Learning Theory and Application at University of Milan.
	His current research interests include mining evolving streaming data, change
	detection, multivariate time series classification, online learning, stream mining,  adaptive learning, and predictive analytics applications.
\end{IEEEbiography}

\end{document}